\documentclass[letterpaper, 10 pt, conference]{ieeeconf}  
\usepackage{amsmath,amssymb,amsfonts}
\usepackage{fancyhdr}
\usepackage{tabularx}
\usepackage{multirow}
\usepackage{graphicx}    
\usepackage{array}
\usepackage{bm}
\usepackage{algorithm}
\usepackage{algorithmicx}
\usepackage{algpseudocode}
\usepackage{hyperref}
\newtheorem{theorem}{Theorem}

\newcommand{\norm}[1]{\left\lVert#1\right\rVert}
\algrenewcommand\algorithmicindent{1.0em}%

\IEEEoverridecommandlockouts                              

\title{\LARGE \bf
DiNNO: Distributed Neural Network Optimization\\ for Multi-Robot Collaborative Learning
}

\author{Javier Yu,$^{1}$ Joseph A. Vincent,$^{1}$ Mac Schwager$^{1}$
\thanks{*This work was funded by NASA ULI grant 80NSSC20M0163, NSF NRI grant 1925030, and NSF NRI grant 1830402.  The first author was also supported on a NSF Graduate Research Fellowship, and the second author was supported on a Dwight D. Eisenhower Transportation Fellowship.}
\thanks{$^{1}$Department of Aeronautics and Astronautics, Stanford University, Stanford, CA 94305, USA, {\texttt\small \{javieryu, josephav, schwager\}@stanford.edu}}%
}

\begin{document}

\maketitle
\thispagestyle{empty}
\pagestyle{empty}

\begin{abstract}
We present a distributed algorithm that enables a group of robots to collaboratively optimize the parameters of a deep neural network model while communicating over a mesh network. Each robot only has access to its own data and maintains its own version of the neural network, but eventually learns a model that is as good as if it had been trained on all the data centrally. No robot sends raw data over the wireless network, preserving data privacy and ensuring efficient use of wireless bandwidth. At each iteration, each robot approximately optimizes an augmented Lagrangian function, then communicates the resulting weights to its neighbors, updates dual variables, and repeats. Eventually, all robots' local network weights reach a consensus.  For convex objective functions, we prove this consensus is a global optimum. We compare our algorithm to two existing distributed deep neural network training algorithms in (i) an MNIST image classification task, (ii) a multi-robot implicit mapping task, and (iii) a multi-robot reinforcement learning task.  In all of our experiments our method out performed baselines, and was able to achieve validation loss equivalent to centrally trained models.  See \href{https://msl.stanford.edu/projects/dist_nn_train}{https://msl.stanford.edu/projects/dist\_nn\_train} for videos and a link to our GitHub repository.
\end{abstract}

\section{Introduction}\label{Introduction}
A group of collaborating robots has the ability to explore, interact with, and experience their environment as a collective much faster than a single robot acting alone. This ability to rapidly gather a large volume and variety of data makes multi-robot systems especially well suited for tasks that involve training deep neural networks using data gathered by the robots. In a cloud robotics scenario, one can imagine thousands of robots networked over a cloud server, able to collectively gather  and process vast volumes of data for a common task (e.g. manipulation, autonomous driving, or human behavior prediction).  In a mesh network scenario, one can similarly imagine a team of robots collaborating to map an environment, learn a control policy, or learn to visually recognize threats in the environment. A central unsolved problem in collaborative robotics, therefore, is how to train neural network models on the robots through local communication such that each robot benefits from the data collected by the entire multi-robot system. 

\begin{figure}[t]
    \centering
    \includegraphics[width=\linewidth]{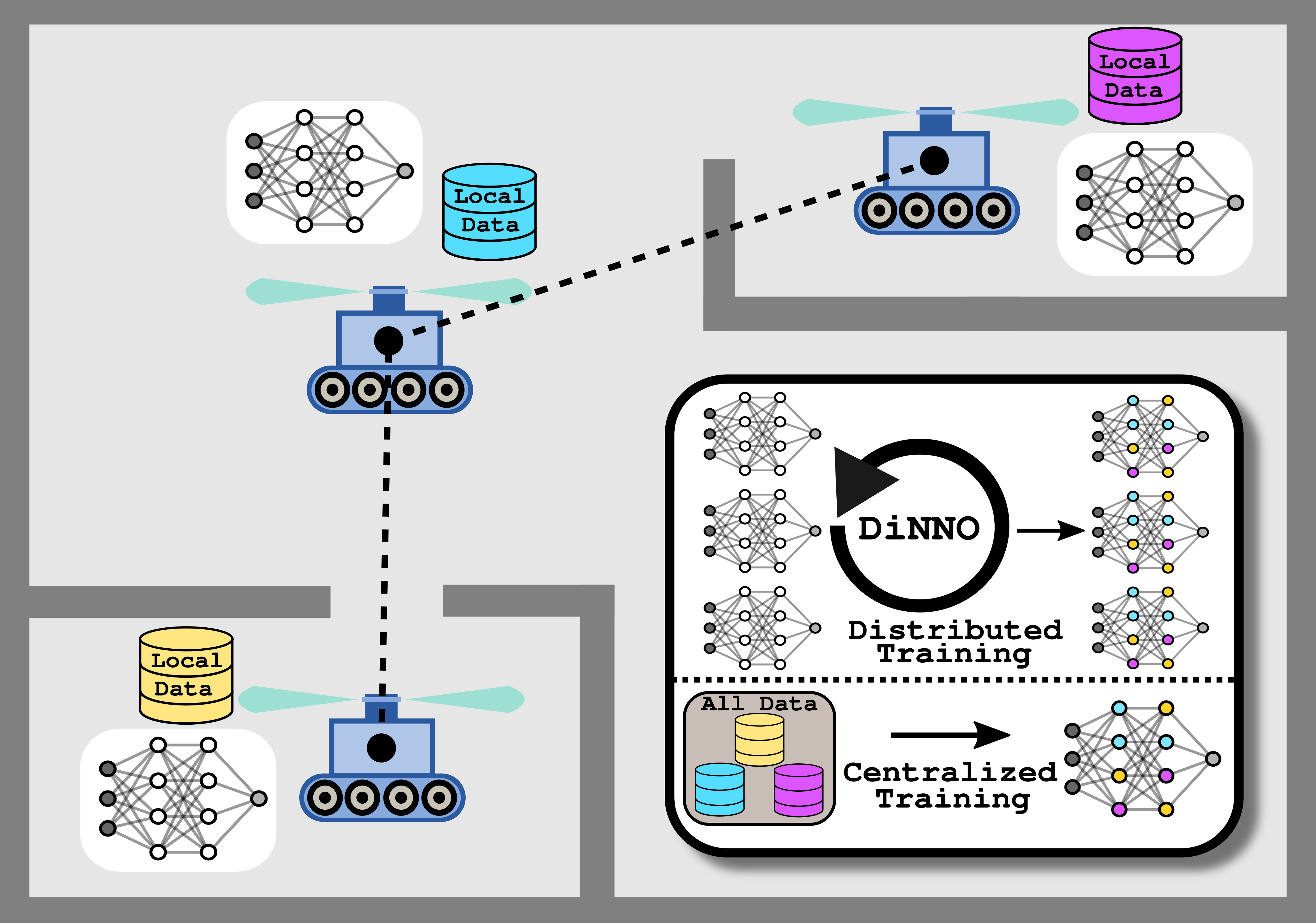}
    \caption{DiNNO allows robots to cooperatively optimize local copies of a neural network model without explicitly sharing data.  In this figure (representative of Section \ref{ex:implicit}), three robots use DiNNO to cooperatively optimize a building occupancy map represented as a neural network. Each robot only sees part of the building, collecting a local lidar data set (colored cylinders).  The robots communicate over a wireless network (dashed lines) to cooperatively optimize their local neural network copies.  The resulting model is as good as if it were trained centrally with all data at once.}
    \label{fig:abstract}
\end{figure}

To solve this problem, we propose Distributed Neural Network Optimization (DiNNO), an algorithm built on the alternating direction method of multipliers (ADMM) \cite{boyd2011distributed}, and that integrates easily with standard deep learning tool sets like PyTorch.  Using DiNNO, robots alternate between local optimization of an objective function, and communication of intermediate network weights over the wireless network.  The robots eventually reach a consensus on their network weights, with each robot learning a neural network that is as good as if it had been trained centrally with the data from all robots, as illustrated graphically in Figure \ref{fig:abstract}.  DiNNO inherits the strong convergence properties of ADMM---for convex objective functions we prove that all robots obtain globally optimal parameters.  However, neural network training is rarely convex.  Using standard deep learning tools within DiNNO we are able to retrieve state-of-the-art deep learning performance, but in a distributed, multi-robot implementation. DiNNO can accommodate time-varying communication graphs (e.g., due to robot motion), and also suits streaming data or online learning, if robots train their network while they move and collect data.  Finally, DiNNO operates by sharing network weights over the communication network, not raw data.  Therefore, robots using DiNNO preserve the privacy and integrity of their own local data set.  This is crucial in scenarios where user data or observations of humans are involved, or when different robot manufacturers need to preserve the privacy of their own data sets. 

A naive approach to solving the multi-robot deep learning problem is to use a mesh network routing protocol to aggregate the data gathered by all of the robots in the system to a single ``leader" robot which then optimizes a deep neural network model, and sends a copy of that trained model back to all of the other robots in the system. We refer to this approach as a ``centralized" solution, and it has a number of distinct drawbacks. First, depending on the size of the gathered data, algebraic connectivity of the communication graph, bandwidth of the communication links, and efficiency of the routing protocol it can take a significant amount of time to aggregate the gathered data at the leader node. A centralized approach is not robust to failure of the leader node, and there are some applications for which it may not be possible to transmit data due to privacy considerations, for instance, due to the European Union General Data Protection Regulation article 46 \cite{EU}.  DiNNO overcomes all of these limitations by enabling leaderless distributed neural network training through local communication among the robots. 

The paper is organized as follows. We give related work in Section \ref{Related Work} and introduce the distributed collaborative learning problem in Section \ref{Problem Formulation}. In Section \ref{Distributed Training} we derive DiNNO starting from a well-known variant of ADMM called Consensus ADMM. In Section \ref{Examples} we present three example robotic deep learning tasks that showcase our method.
\section{Related Work}\label{Related Work}
Learning has been used to address a variety of problems for multi-robot systems and is becoming more popular as deep learning tools are made more accessible. A deep learned controller is used to model multi-quadrotor interactions in \cite{shi2020neural} and \cite{shi2021neural}, and is shown to considerably outperform traditional non-linear controllers. Reinforcement learning can also be a useful tool in multi-robot contexts and \cite{chen2017decentralized} shows that an offline learned value function can be used to perform real-time and uncertainty aware collision avoidance. The collective training methods in \cite{yahya2017collective} and \cite{sartoretti2019distributed} demonstrate how experience aggregation from multiple robots can speed up policy optimization. While not deep learning, Gaussian processes are another popular learning tool that have been used for various regression tasks with data collected online by multi-robot systems \cite{luo2018adaptive, habibi2021human}. These works all showcase applicability of collaborative learning models for multi-robot systems, but in general, aside from \cite{habibi2021human}, do not provide a distributed framework from which to perform this learning. 

The problem of training neural networks in a distributed way using data aggregated from individual robots can be viewed as a specific instantiation of a more general class of problems referred to as distributed optimization problems. Distributed optimization is the study of algorithms for solving optimization problems where a sum of individual objective functions, which correspond in this case to the individual robots, is optimized using local computation and message passing. This formulation was first proposed in \cite{tsitsiklis1984problems}, and has been of renewed interest since the seminal work \cite{nedic2009distributed} which presented distributed subgradient descent for convex distributed optimization problems. Subsequent research has focused on improving convergence rates \cite{shi2015extra} and extending analysis to a broader range of problems including time-varying communication graphs \cite{nedic2014distributed}, and online or streaming convex objectives \cite{hosseini2016online}. For an overview of the broader distributed optimization literature we direct the reader to the surveys \cite{halsted2021survey, nedic2018network, yang2019survey}.

Our algorithm is designed for networked robots collaboratively optimizing a neural network model using data local to each robot.  Specifically, this task requires distributed optimization of nonconvex deep learning loss functions in situations where the communication graph can be time-varying, and where data may be streaming as it is collected by the robots. While, to our knowledge, no existing works formally address all of these requirements, there are some works in the distributed optimization literature that address general nonconvex distributed optimization objectives, and some use simple distributed neural network training problems as motivating examples.

In \cite{lian2017can} the distributed subgradient descent algorithm is extended to stochastic gradients as the distributed stochastic gradient descent (DSGD) algorithm, and uses training of a CIFAR-10 classification model as a benchmark problem. One approach to improving convergence rates in distributed optimization is to introduce an auxiliary variable that estimates the global gradient. Several works make use of this mechanism, and extend it to the domain of nonconvex optimization with stochastic gradients \cite{pu2021distributed, lu2019gnsd}. The Choco-SGD algorithm for distributed deep learning \cite{koloskova2019decentralized} is another algorithm similar to DSGD with the variations that it uses a gossip mechanism for improved consensus, and incorporates a quantization step for reducing communication bandwidth. The authors of \cite{di2016parallel} propose using local convex approximations, based on global gradient estimates found through gradient tracking, for distributed neural network optimization and show basic results on small networks. 

The edge consensus learning algorithm proposed in \cite{edge_consensus} is similar to our approach in that it is derived from ADMM, but instead of addressing the nonconvex primal update directly it uses a linearization similar to that proposed in \cite{ling2015dlm}, which results in a gradient descent like update. While still technically a primal-dual method, the update equations do not include a local optimization procedure, and are more similar to primal domain methods like DSGD (see \cite{ling2015dlm}, Remark 1). A number of other nonconvex distributed optimization methods are discussed in the survey \cite{chang2020distributed}. 

Our proposed algorithm for distributed multi-robot deep learning takes full advantage of existing deep learning tools and optimizers, and we demonstrate superior performance compared to benchmarks on robotics related deep learning tasks such as neural implicit mapping and deep multi-agent reinforcement learning. 

\section{Problem Formulation}\label{Problem Formulation}
We consider deep learning problems where portions of a data set, $\mathcal{D}$, are sent to, or collected by, $N$ robots that operate in a connected communication graph $\mathcal{G} = (\mathcal{V}, \mathcal{E})$. Let $\mathcal{D}_i$ be the portion of local data that belongs to robot $i \in \mathcal{V}$ where the union of all the local data sets, $\mathcal{D}_i$, is the joint data set, $\mathcal{D}$. In some cases $\mathcal{D}_i$ can represent access to a time-varying data set gathered from an incoming private data-stream (as in Section \ref{ex:implicit}). 

The model we would like to optimize has the form $y = f(x; \: \theta)$. Specifically we consider $f$ to be a deep neural network that implements a continuous function $f(x) : \mathbb{R}^{n} \rightarrow \mathbb{R}^{m}$ to give a map from inputs $x \in \mathbb{R}^n$ to outputs $y \in \mathbb{R}^m$. The neural network is parameterized by network weights $\theta \in \mathbb{R}^d$. We make no special assumptions about the architecture (e.g. feed-forward, convolutional, residual, etc.). We can then formalize our distributed learning optimization problem as
\begin{align}
    \underset{\theta \in \mathbf{R}^d}{\text{minimize}} \quad \sum_{i \in \mathcal{V}} \ell(\theta; \mathcal{D}_i) \label{eq:joint_loss}
\end{align}
where $\ell(\cdot)$ is the objective function (loss function) which is generally nonconvex and often nonsmooth (due to ReLU activation). Common deep learning tasks such as classification, regression, and unsupervised learning have different objective functions and a distributed deep learning optimizer should be general enough to achieve good performance across all of these problems.

Suppose that the decision variable, $\theta$, is separated such that each robot maintains their own instance of it, $\theta_i \in \mathbb{R}^d$. This yields the equivalent optimization problem
\begin{subequations}
\begin{align}
    \underset{\theta \in \mathbf{R}^d}{\text{minimize}}& \quad \sum_{i \in \mathcal{V}} \ell(\theta_i; \mathcal{D}_i) \\
    \text{subject to}& \quad \theta_i = \theta_j \quad \forall (i, j) \in \mathcal{E}. \label{eq:consensus}
\end{align}
\label{eq:main}
\end{subequations}
This optimization problem is amenable to a distributed solution in which robots minimize local objective functions, and take additional steps to come to agreement (consensus) on the value of the decision variable. Replacing the data defined loss functions in (\ref{eq:main}) with arbitrary objective functions yields the general formulation of a distributed optimization problem. Though we consider distributed deep learning which is typically unconstrained, a constrained formulation of (\ref{eq:main}) covers many other robotics problems including distributed target tracking, coordinated package delivery, and cooperative multi-robot mapping \cite{halsted2021survey}.

\section{Distributed Training}\label{Distributed Training}
A standard method for solving convex distributed optimization problems is the consensus alternating direction method of multipliers (CADMM). CADMM is an ADMM-based optimization method where compute nodes (robots) alternate between updating their primal and dual variables and communicating with neighboring nodes. To achieve a distributed primal-dual update, CADMM introduces auxiliary primal variables (i.e. $\theta_i = z_{ij}$ and $\theta_j = z_{ij}$ instead of $\theta_{i} = \theta_{j}$). CADMM works by first optimizing the auxiliary primal variables, followed by the original primal variables, then the dual variables, as in the original formulation of ADMM \cite{boyd2011distributed}. Implementations of CADMM then perform minimization with respect to the primal variables and gradient ascent with respect to the dual on an augmented Lagrangian that is fully distributed among the robots:
\begin{align}
    \mathcal{L}_a = \sum_{i \in \mathcal{V}} \ell(\theta_i) + p_i^{\top}\theta_i + \frac{\rho}{2} \sum_{j \in \mathcal{N}_i} \vert \vert \theta_i - z_{ij} \vert \vert _2 ^2
\end{align}
where $p_i$ represents the dual variable that enforces agreement between node $i$ and its communication neighbors, and $\mathcal{N}_i$ is the set of indices for neighboring nodes of $i$. The parameter $\rho$ that weights the quadratic terms in $\mathcal{L}_a$ is also the step size in the gradient ascent of the dual variable. Furthermore, the algorithm can be simplified by noting that the auxiliary primal variable update can be performed implicitly ($z_{ij} = \frac{1}{2}(\theta_i + \theta_j)$). Initializing the dual variables at zero then yields the following distributed update equations for CADMM:
\begin{subequations}
\begin{align}
       p_i^{k + 1} &= p_i^{k} + \rho \sum_{j \in \mathcal{N}_i} (\theta_i^{k} - \theta_j^{k}) \label{eq:dual}\\
       \theta_i^{k + 1} &= \underset{\theta}{\text{argmin}} \: \: \ell (\theta; \mathcal{D}_i) + \theta^{\top}p_i^{k+1} \label{eq:primal}\\
       & \qquad \qquad + \rho \underset{j \in \mathcal{N}_i}{\sum}\norm{\theta - \frac{\theta_i^{k} + \theta_j^{k}}{2}}_2^2. \nonumber
\end{align}
\end{subequations}
Typically, the primal variables are initialized uniformly to an initial guess $\theta_i^0 = \theta_{\text{initial}}$. This derivation of CADMM is addressed in much finer detail in \cite{chang2014multi}.

In our multi-robot learning problem, the objective function for the primal update (\ref{eq:primal}) if performed by each robot using CADMM is composed of three terms: a neural network loss on the robot's local data, a linear term from the dual variable, and a regularization term. It is obvious that applying CADMM directly to the neural network training problem results in intractable primal updates due to the neural network loss component. The key insight, which we use in our algorithm DiNNO, is that this primal optimization can be performed approximately, stopping well before convergence to a local optimum. Formally, we propose replacing the exact minimization of the primal update, (\ref{eq:primal}), with an approximate solution found by taking a small number of steps, $B$ (typically between 2 and 10), of a stochastic mini-batch first order method (SFO) on the entire primal objective function. Our proposed algorithm is shown in Algorithm \ref{alg:ours} with the approximate primal update performed in lines 12 - 16. We replace the current primal iterate with $\psi$ in order to avoid including two iteration count super scripts, and let $G(\psi^\tau;\rho, p_i^{k + 1}, \theta_i^k, \{\theta_j^k\}_{j \in \mathcal{N}_i}, \mathcal{D}_i)$ represent the step taken by a SFO on the objective in the primal update. To be clear, $G$ computes a stochastic gradient from a mini-batch in $\mathcal{D}_i$ not a gradient on the entire local data set.

In some of our experiments we found it beneficial to also add a ``scheduled" increase (Algorithm \ref{alg:ours}, line 9) for the penalty parameter $\rho$ in similar fashion to the learning rate schedules used in deep learning. For notational simplicity, we overload the variable $\rho$ to also mean this schedule of parameter values, and make explicit note of all cases where one is used. Although generally we leave this term constant, it can be useful to gradually increase it when faster consensus is desired. This schedule can be provided to robots prior to optimization, and does not compromise the distributed nature of the DiNNO algorithm.

An added benefit of our approach for distributed deep learning is that it pairs well with existing deep learning libraries (e.g. PyTorch \cite{pytorch}) because the approximate primal minimization can be performed with minimal changes to the typical training loops used to optimize individual neural networks. We find that this is beneficial because automatic differentiation and state-of-the-art neural network optimizers, like Adam \cite{kingma2014adam}, can be used to perform the approximate primal update, and practitioner knowledge from experience training individual neural networks is transferable.

\begin{algorithm}
\caption{{\small Distributed Neural Network Optimization (DiNNO)}}
\label{alg:ours}
\begin{algorithmic}[1]
\State \textbf{Require:} $\ell(\cdot), \:\theta_{initial}, \: \mathcal{G}, \: \mathcal{D}, \: \rho$
    \For{$i \in \mathcal{V}$} \Comment{Initialize the iterates}
        \State $p_i^{0} = 0$ \Comment{Dual variable}
        \State $\theta_i^{0} = \theta_{initial}$ \Comment{Primal variable}
    \EndFor
    \State 
    \For{$k \gets 0 \text{ to } K$} \Comment{Main optimization loop}
        \State \textbf{Communicate:} send $\theta_i^{k}$ to neighbors $\mathcal{G}$
        \For{$i \in \mathcal{V}$} \Comment{In parallel}
            \State $p_i^{k + 1} = p_i^{k} + \rho \sum_{j \in \mathcal{N}_i} (\theta_i^{k} - \theta_j^{k})$
            \State $\psi^{0} = \theta_i^k$
            \For{$\tau \gets 0 \text{ to } B$} \Comment{Approximate primal}
                \State $\psi^{\tau + 1} = \psi^{\tau} + G(\psi^{\tau};\rho, p_i^{k+1}, \theta_i^k, \{\theta_j^k\}_{j \in \mathcal{N}_i}, \mathcal{D}_i)$ \label{alg:line:primal}
            \EndFor
            \State $\theta_i^{k+1} = \psi^{B}$ \Comment{Update primal}
        \EndFor
    \EndFor  
    \State \Return{$\{\theta_i^{K}\}_{i \in \mathcal{V}}$}
\end{algorithmic}
\end{algorithm}

\subsection{Convergence Properties}
\begin{theorem}[Optimality of Algorithm \ref{alg:ours}]
Let each local objective function $l(\theta, \mathcal{D}_i)$ be strongly convex and $L$-smooth. Furthermore, let 
\begin{align*}
    G = -\frac{1}{L} \nabla_\theta \big(l(\theta, \mathcal{D}_i) + \theta^{\top}p_i^{k+1} + \rho \underset{j \in \mathcal{N}_i}{\sum}\norm{\theta - \frac{\theta_i^{k} + \theta_j^{k}}{2}}_2^2\big).
\end{align*}
Then Algorithm \ref{alg:ours} converges to the unique global solution with linear convergence rate.
\end{theorem}
\begin{proof}
Given strongly convex and $L$-smooth local objectives, each primal update (Algorithm \ref{alg:ours} line \ref{alg:line:primal}) converges to the global solution with a linear convergence rate as shown in \cite{gd_convergence}. Given globally optimal primal updates and the stated assumptions, Algorithm \ref{alg:ours} is a special case of the decentralized ADMM algorithm studied in \cite{convergence} where it was shown to have linear convergence to the global solution.
\end{proof}

Clearly, in deep learning problems global solutions and linear convergence rates are not ensured due to nonlinear neural networks creating nonconvex, and often nonsmooth, objective functions. Nevertheless, we find that in practice Algorithm \ref{alg:ours} performs well for solving distributed deep learning problems, and converges to solutions similar in quality to those found by centralized optimization.

\subsection{Baseline Algorithms}
In Section \ref{Examples} we show that DiNNO is an extremely effective method for distributed training of neural network models. We compare DiNNO against two other commonly referenced stochastic first order distributed optimization methods: distributed stochastic gradient descent (DSGD) \cite{lian2017can}, and distributed stochastic gradient tracking (DSGT) \cite{pu2021distributed}. Like DiNNO, both methods have each robot maintain a local copy of the optimization variable (neural network weights), and use message passing and locally computed stochastic gradients to collaboratively optimize the neural network. DSGD uses the following simple update equation
\begin{align}
    \theta_i^{k + 1} = \underset{j \in \mathcal{V}}{\sum} w_{ij} \theta_j^{k} - \alpha^{k} g(\theta_i^{k}) \label{eq:dsgd}
\end{align}
where $w_{ij}$ is an element of a doubly stochastic matrix $W$ that has a sparsity pattern matching that of the graph Laplacian of $\mathcal{G}$, $\alpha^{k}$ is a decaying step size, and $g(\theta_i^{k})$ is a stochastic (or mini-batch) gradient of $\ell (\theta_i^k; \mathcal{D}_i)$. While (\ref{eq:dsgd}) may not at first appear to be a distributed algorithm, the sparsity pattern of $W$ means that each node only needs $\theta_j^k$ from its immediate neighbors to compute its update step.

The updates for DSGT are similar to those of DSGD, but an additional auxiliary variable is added to estimate the gradient of the joint loss:
\begin{align}
   \theta_i^{k+1} &= \underset{j \in \mathcal{V}}{\sum} w_{ij} (\theta_j^k - \alpha y_j^k) \\ 
   y_i^{k+1} &= \underset{j \in \mathcal{V}}{\sum} w_{ij} y_j^k + g(\theta_i^{k + 1}) - g(\theta_i^k).
\end{align}
An important point to note is that for DSGT the message size sent at each communication round is double that of both DSGD and DiNNO which only send $\theta_i^k$. For DSGT and DSGD we use the Metropolis-Hastings weights as $W$.

Alternative benchmark algorithms include (\cite{lu2019gnsd, koloskova2019decentralized, di2016parallel, edge_consensus, chang2020distributed}) but, in general, they share many core characteristics with the proposed baselines DSGT and DSGD.

\begin{figure*}[t]
    \includegraphics[width=\textwidth]{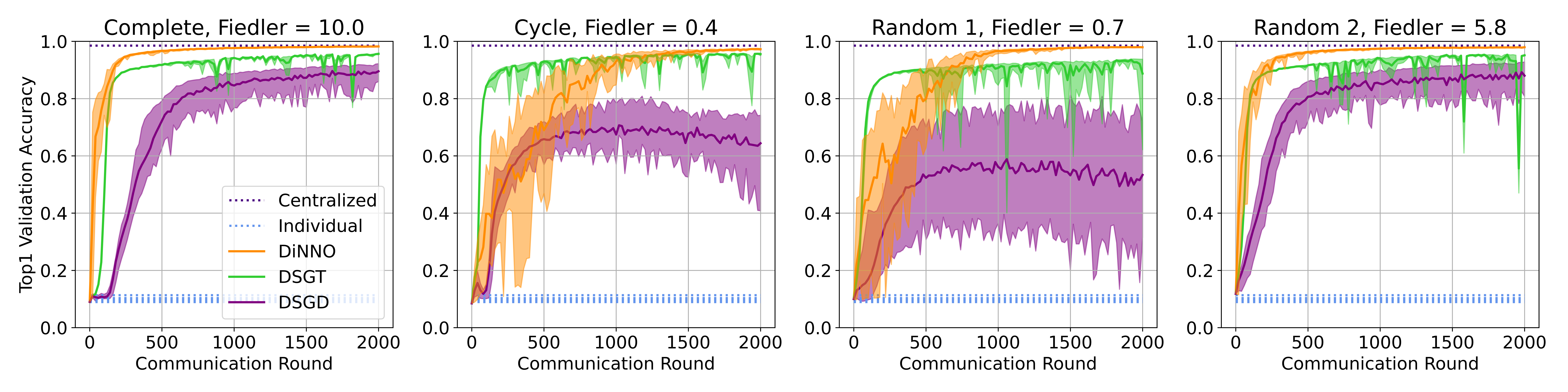} 
    \centering
    \caption{Each plot shows the Top-1 accuracy of the neural network models on the validation set of the MNIST problem. For each algorithm we evaluate the local neural networks stored by each robot on the validation set. Solid lines show the average validation loss across all robots at the current communication round, and filled areas are upper and lower bounded by the best and worst performing networks for each particular algorithm. The four plots show algorithm performance on four different communication graphs (using the same hyperparameters across all instances). In all cases our method, DiNNO, converges to the same accuracy as the centralized solution, and outperforms both of the baseline methods.}
    \label{fig:mnist}
\end{figure*}

\subsection{Data Distributions}
The way in which local data is partitioned between the robots strongly influences the convergence rate of distributed deep learning methods. For example, in classification, problems where each robot has access to a subset of examples from all classes are easier to solve with distributed optimization than problems in which each robot only has access to labelled data for a single class. We refer to these two data distributions as \textit{homogenous} and \textit{heterogenous} respectively. In homogeneous classification a robot which optimizes directly on its local data set without communication may be able to achieve a relatively high validation accuracy. Whereas a robot in the heterogeneous case is unlikely to achieve a high accuracy on any class other than what it observed.

\section{Experiments}\label{Examples}
In the following three examples we demonstrate that DiNNO can be applied to a wide range of multi-robot learning applications, and demonstrates a substantial improvement against baseline distributed optimization algorithms. In each example we compare with two other common distributed optimization methods: DSGD and DSGT. We implement DiNNO, DSGD, and DSGT in a general framework such that for each experiment the optimization algorithm is unchanged, but a different objective function and data set are provided. Hyperparameter values are reported in Section \ref{Appendix}.
\subsection{MNIST Classification} \label{ex:mnist}
To clearly illustrate the potential for Algorithm \ref{alg:ours} to train a shared neural network from disparate data observers we first consider the well known MNIST classification problem \cite{lecun1998mnist}. In this problem a neural network must learn to classify images of handwritten digits to their correct integer values. We train a network composed of a convolutional layer with three 5x5 filters followed by 2 hidden linear layers of width $576, 64$ with ReLU activation and an output layer with log-softmax activation. We use the negative log-likelihood loss function. Each robot only has access to labelled digits from a single class. In this experiment, we perform distributed training with this heterogeneous data distribution on four different communication topologies: a complete graph, a cycle graph, and two random graphs. 

In Figure \ref{fig:mnist} we show the average, worst, and best Top-1 accuracy for each method on the distributed MNIST classification problem. Also included is the centralized result (98.5\% validation accuracy), and the individual results, which as expected have roughly 10\% validation accuracy. We observe that while DSGT quickly trains to good accuracy in these problems, DiNNO achieves much better final accuracy as training progresses with lower variance in later iterations. DSGD has relatively poor performance with high variance. The Fiedler value (algebraic connectivity) of each graph is reported to emphasize that DiNNO is capable of training to centralized accuracy with highly heterogeneous data distributions, and low graph connectivity.

\subsection{Neural Implicit Mapping} \label{ex:implicit}
In robotics and computer vision there is growing interest in using neural networks to represent functions which implicitly define the geometry and visual features of an environment \cite{nerf, imap}. In their basic form, implicit density field networks take as input an $(x,y,z)$ spatial coordinate and output a single density value between 0 and 1. Such networks are able to represent complicated 3D scenes in a single memory-efficient function. In this example we use DiNNO to learn the density field of a two dimensional environment where data collection and computation is distributed across multiple robots. The robots also have access to a global coordinate frame which enables cooperative mapping, but a future line of research would be to implement this same pipeline in conjunction with a distributed pose optimization algorithm.

The environment we seek to map is a 2D building floorplan environment from the CubiCasa5K data set \cite{kalervo2019cubicasa5k}. Seven robots are deployed, and each robot gathers data from the environment by collecting lidar scans as it traverses a closed loop, precomputed trajectory. To simulate data streaming the robots update their local networks at regular intervals from datasets of their last 400 collected lidar scans (an entire trajectory has 3000-4000 scans). Figure \ref{fig:true_map} shows the ground truth environment with seven robot paths and one lidar scan. There is some overlap in the locations traversed by each robot, but many locations, especially at the edges and corners of the map, are only viewed by a single robot.

\begin{figure}
    \includegraphics[width=\linewidth]{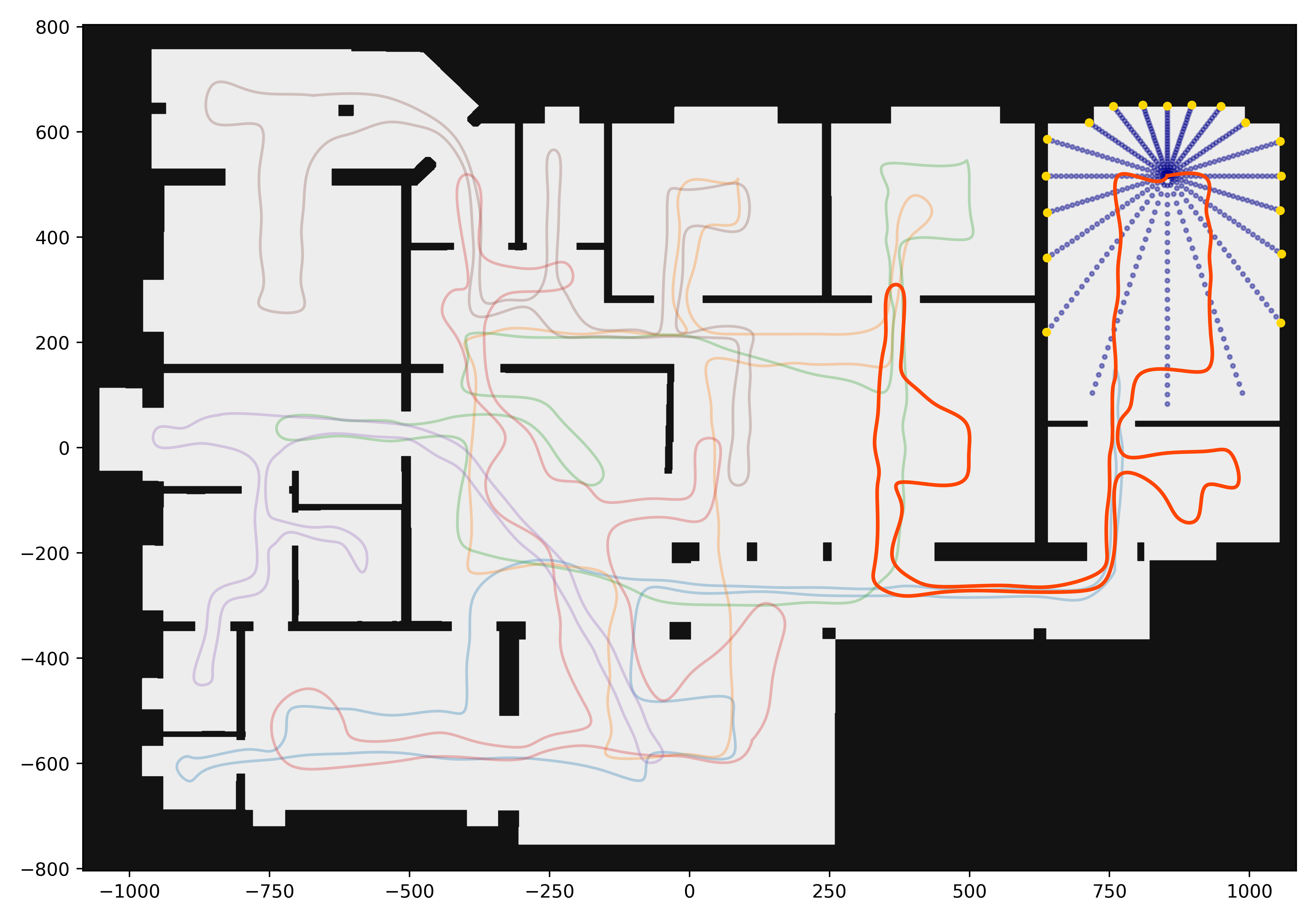} 
    \centering
    \caption{Ground truth map. Highlighted is a single robot's trajectory and a single lidar scan is shown with high (gold) and low (blue) density points.}
    \label{fig:true_map}
\end{figure}

We train a feedforward network with four hidden layers of size $256, 64, 64, 64$ where the first hidden layer has sinusoidal activation, the remaining hidden layers have ReLU activation, and the output layer has sigmoid activation to restrict our density estimates to $(0,1)$. The sinusoidal activation of the first hidden layer is common in implicit mapping networks and inspired by Fourier Feature Networks \cite{tancik2020fourier}. For a loss function we use the binary cross entropy between the sampled and predicted density. 

Our validation set is composed of novel lidar scans from uniformly sampled locations across the entire map, and this ensures that the validation loss reflects loss only on areas where the robots have been able to gather data (e.g. not inside walls). For the communication graph at each round an edge between two robots exists if their pairwise distance is below a certain threshold (1500 in this experiment). The constant motion of the robots results in a communication graph that is dynamic (in terms of edges), but always connected.

Figure \ref{fig:mapping_loss} shows the validation loss for our method as well as DSGD and DSGT. Of the tested methods DiNNO best minimizes the validation loss, once again approaching the performance of centralized training. DSGD and DSGT train less effectively, and converge to a poor quality solution.

\begin{figure}
    \includegraphics[width=\linewidth]{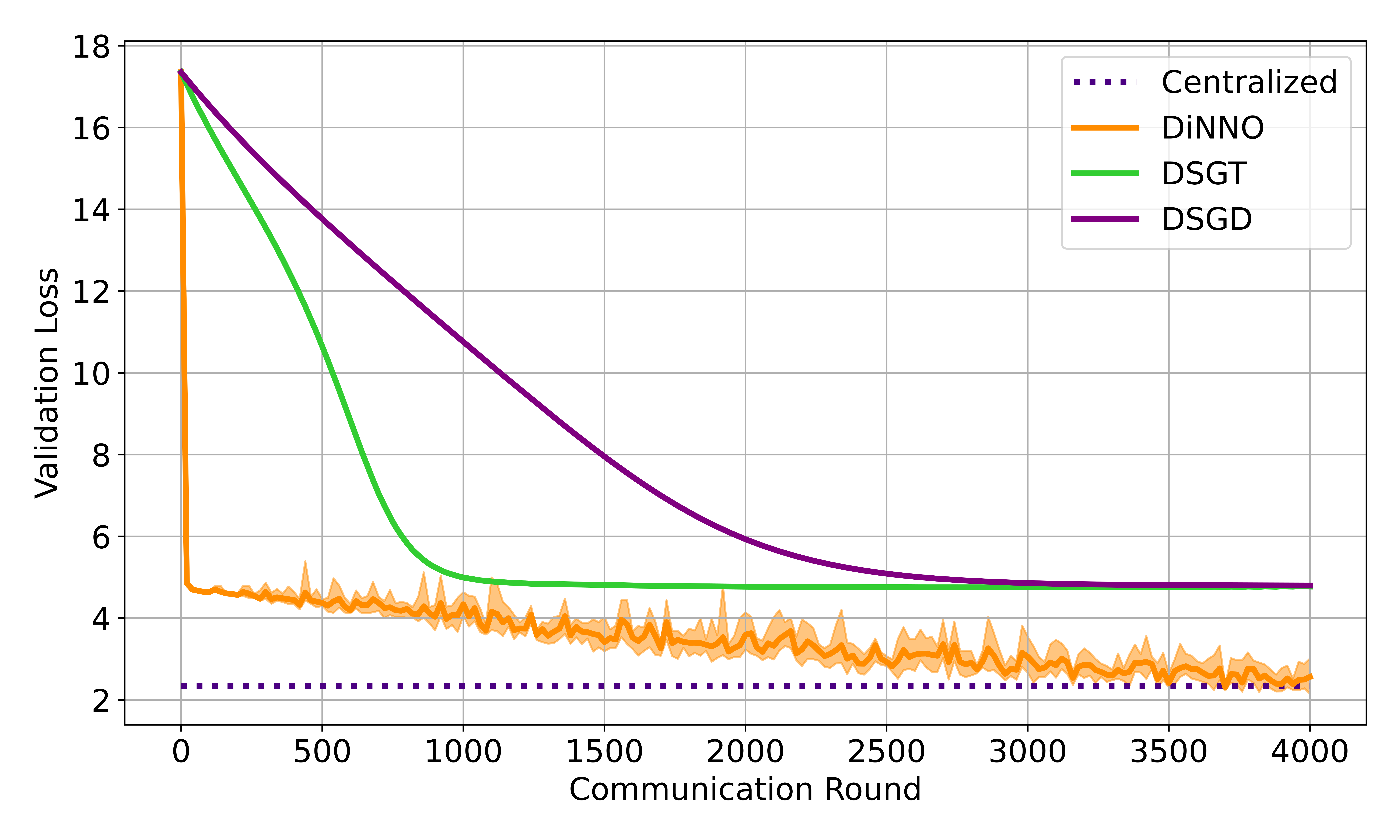}
    \centering
    \caption{Validation loss versus communication iteration for the neural implicit mapping experiment. Both baseline algorithms DSGT and DSGD appear to converge to a poor quality minima while DiNNO (ours) converges to a model with validation loss matching that of the centralized solution.} 
    \label{fig:mapping_loss}
\end{figure}

Figure \ref{fig:reconstruction} shows the map learned by each method as well as learned maps from individual robots training on only their own data. Each of the robots when using DiNNO is able to provide a faithful reconstruction of the ground truth environment while never having traversed its entirety or having received any raw data from other robots. When tested with DSGD and DSGT robots converge to similar local minima which result in incoherent maps. 

To verify the performances of DSGT and DSGD we reran this experiment several times, and both methods always converged to poor performing local minima. Additionally, we emphasize that the simulation code for these two methods is unchanged between this experiment and the MNIST one above where both methods are able to learn acceptable classifiers. We speculate that this is a challenging problem where only a small amount of suboptimality is allowable to achieve a useful representation, and DSGT and DSGD may be unable to either fine tune their neural network weights, escape poor local minima, or handle streaming data.

\begin{figure*}
    \includegraphics[width=\textwidth]{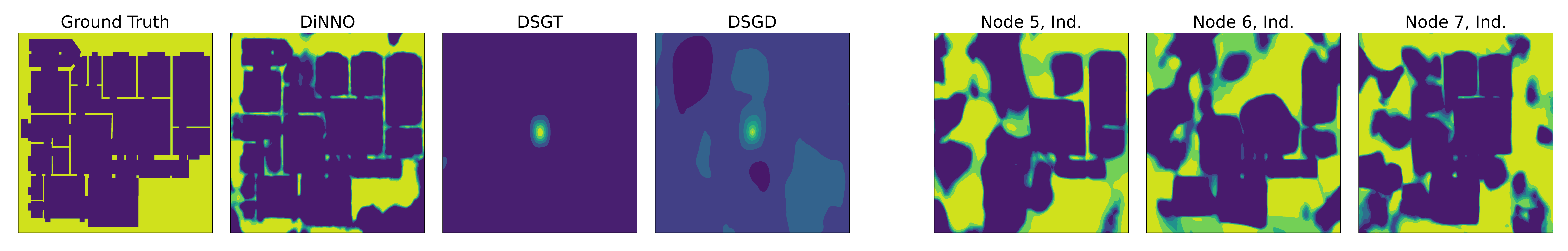} 
    \centering
    \caption{The left most plot shows the ground truth density map, and moving right the next three plots are the reconstructions from the neural implicit maps found by the three tested distributed algorithms. Here DiNNO is the only method that is able to learn a coherent map. The reconstructions were produced by querying the optimized (and agreed upon) networks on a grid mesh of points on the map. The last three plots show reconstructions produced from three of the seven robots when communication is not used (training exclusively on local data with Adam for 10 epochs). Since these robots do not have information from other areas on the map they are only able to reconstruct regions which those robots have traversed.}
    \label{fig:reconstruction}
\end{figure*}
\subsection{Multi-Agent Reinforcement Learning} \label{ex:marl}
For the final example we show that our algorithm can be used for deep multi-agent reinforcement learning (deep MARL). Specifically we use DiNNO for distributed learning of a decentralized policy applied to a standard continuous state and action, multi-robot, predator-prey problem that was first introduced in \cite{mpe}. MARL is known to be an especially hard learning task due to the inherent \textit{nonstationarity} of the environment. That is, the environment changes during learning because the other learning agents also have evolving policies. For background on MARL and a recent survey of deep MARL algorithms see \cite{param_sharing} and \cite{deep_marl}.

In our learning environment three robots must work together to pursue a faster evader robot in the presence of stationary randomly placed obstacles, as shown in Figure \ref{fig:tag}. Implemented in PettingZoo \cite{pettingzoo}, the environment operates according to the Actor Environment Cycle (AEC) Game model in which pursuers make observations, take actions, and receive rewards sequentially before the environment as a whole is updated. The pursuers are homogeneous with actions $a = $ \texttt{[none, right, left, up, down]} $\in \mathbb{R}^5 : \mathcal{A}$ and observations $o = $\texttt{[self\_vel, self\_pos, other\_pursuers\_rel\_pos, evader\_rel\_pos, evader\_rel\_vel]} $\in \mathbb{R}^{12} : \mathcal{O}$. Actions are clipped to be on the interval $[0,1]$. The evader obeys a heuristic policy, moving in the direction opposite of the position of the nearest pursuer. However, to prevent unfair evasion, the evader cannot propel itself outside a square of radius $1.2$. The reward function penalizes pursuing robots based on their distance from the evader, and pursuers receive a large positive reward for tagging (capturing) the evader.
\begin{figure}[ht!]
    \includegraphics[width=\linewidth]{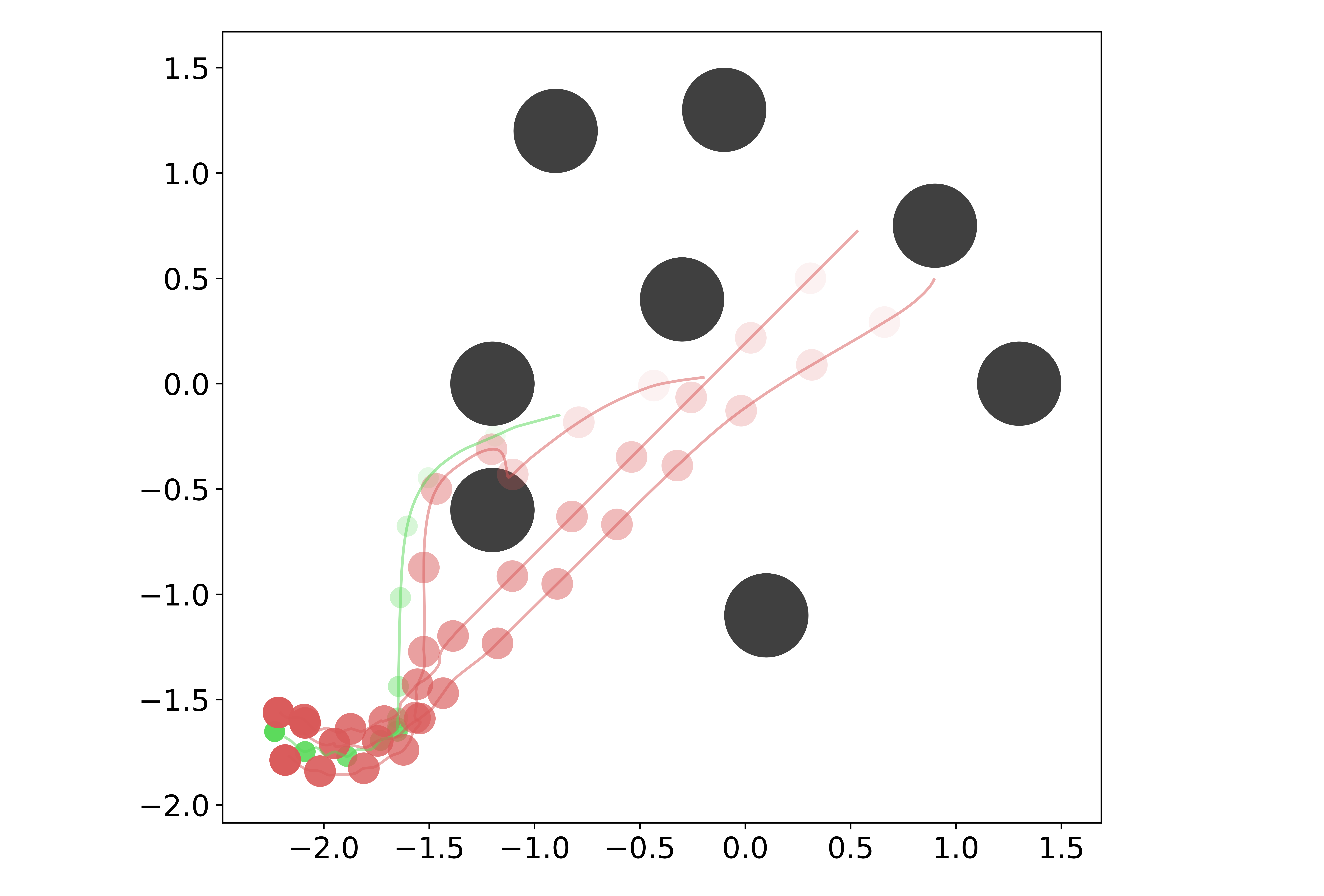} 
    \centering
    \caption{A decentralized policy rollout in the predator-prey environment. Pursuers (red), using policies learned with DiNNO applied to PPO, attempt to capture a faster evader (green) in the presence of obstacles (black).}
    \label{fig:tag}
\end{figure}

To solve this problem we extend the PPO algorithm \cite{ppo} with DiNNO to train a shared, decentralized policy. PPO uses an actor-critic scheme, and with DiNNO the policy (composed of two networks) can be optimized in a distributed way. At consensus the robots all converge to the same policy in accordance with a parameter sharing approach which has been shown to be effective for many MARL problems \cite{gupta_2017, param_sharing}. Typically the policies for parameter sharing are trained by some centralized compute node that aggregates the experiences of each of the robots. Applying DiNNO to PPO results in a relatively unexplored paradigm for MARL where both training and execution are fully distributed.

In this example the actor and critic networks are feedforward ReLU networks with 3 hidden layers of 64 neurons each. The robots simulate communication through a fully connected graph, and robots update their policies using individually collected data every 10 episodes. Results from this experiment are shown in Figure \ref{fig:RL_res}. For each training scenario (DiNNO, DSGD, DSGT, centralized) we show the mean of the average episodic reward achieved by the multi-robot predator team as training progresses. To verify training quality of each algorithm, we run the training scenario five times for each algorithm to generate the shown statistics. 

As it does in previous experiments DiNNO achieves the same average episodic reward as a policy trained using PPO with experiences aggregated from all three of the robots. Both DSGD and DSGT seem unable to learn a policy that results in positive episodic reward. 

\begin{figure}[ht!]
    \includegraphics[width=\linewidth]{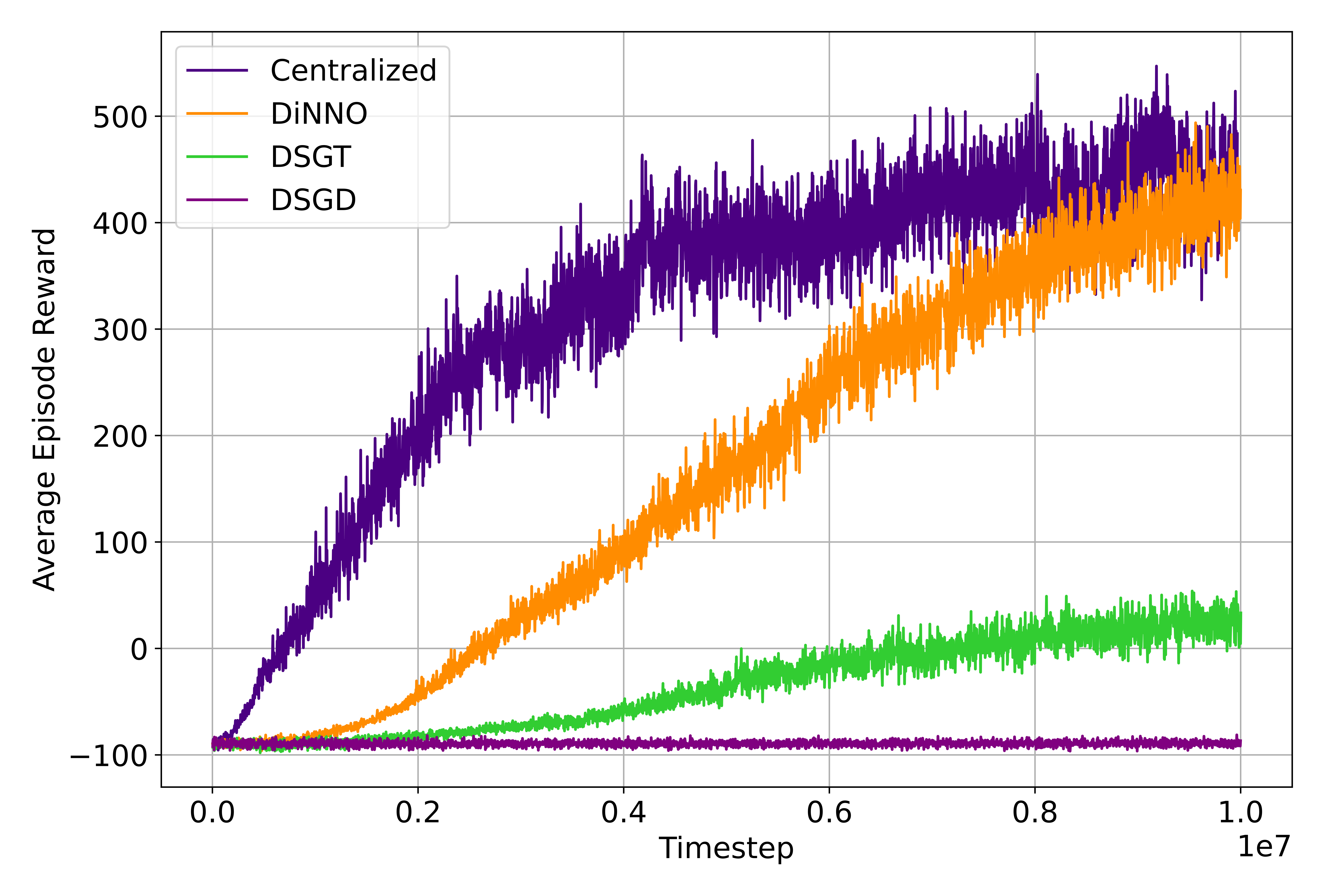} 
    \centering
    \caption{Episodic reward (averaged across 10 episodes per network update) vs environment time steps (summed across all episodes) for DiNNO, DSGD, DSGT, and centralized. Shown results are averaged across 5 training runs to show training robustness. DiNNO is the only algorithm to achieve good performance, matching centralized performance after 10 million time steps.}
    \label{fig:RL_res}
\end{figure}

\section{Conclusion}\label{Conclusion}
We presented the DiNNO algorithm that enables high performance distributed training of deep neural networks, and showcase its versatility on three diverse multi-robot learning tasks. In comparisons to existing distributed learning methods our algorithm consistently achieves better validation performance and converges to performance of centrally trained models. Directions for future work include learning neural implicit density functions from real 3D depth data, and exploring the capabilities of DiNNO for more complex distributed reinforcement learning tasks. 




\section*{Appendix}\label{Appendix}
\subsection{MNIST Experiments} \label{app:mnist}
Hyperparameters used across all four graphs were the same. DiNNO uses $B = 2$, $\rho_0 = 0.5$ increasing 0.3\% per communication round, and Adam as its primal optimizer with a log learning rate schedule (0.005 - 0.0005) for the primal update. DSGT uses $\alpha = 0.005$. DSGD uses a decaying stepsize following $\alpha^{k+1} = \alpha^{k} (1 - \mu \alpha^k)$ where $\alpha^0 = 0.005$ and $\mu = 0.001$. All methods use batch size 64.


\subsection{Neural Implicit Mapping Experiment}\label{app:implicit}
DiNNO uses $B = 5$, $\rho^0 = 0.1$ increasing 0.3\% per communication round, batch size of 10,000, and Adam with a log learning rate schedule of (0.001 - 0.0001). DSGT uses $\alpha = 0.001$, and a batch size of 20,000. DSGD uses $\alpha^0 = 0.001$, $\mu = 0.001$, and a batch size of 20,000.

\subsection{Multi-agent Reinforcement Learning}\label{app:marl}
For each algorithm we use the following hyperparameters: 200 steps per episode, 2000 steps between actor/critic network updates, reward discount factor $\gamma=0.99$, and PPO clipping parameter $0.2$. We allow each algorithm 5 gradient steps ($B=5$) per batch of data to update actor and critic networks. The actor learning rates for DiNNO, DSGD, and DSGT are $0.0003$, $0.001$, and $0.01$, respectively. Only DSGT has a separate critic learning rate, $0.00001$, due to exploding gradients otherwise. For DiNNO we set a constant $\rho = 1.0$.



\bibliographystyle{./IEEEtran} 
\bibliography{./IEEEabrv,./IEEEexample}

\end{document}